\documentclass{article}
\pdfoutput=1

     \usepackage[nonatbib,preprint]{neurips_2019}

\usepackage[utf8]{inputenc} % allow utf-8 input
\usepackage[T1]{fontenc}    % use 8-bit T1 fonts
\usepackage{hyperref}       % hyperlinks
\usepackage{url}            % simple URL typesetting
\usepackage{booktabs}       % professional-quality tables
\usepackage{amsfonts}       % blackboard math symbols
\usepackage{nicefrac}       % compact symbols for 1/2, etc.
\usepackage{microtype}      % microtypography
\usepackage{amsmath}
\usepackage{amssymb}
\usepackage{amsthm}
\usepackage{subfigure}
\usepackage{graphicx}
\usepackage{algorithm}
\usepackage{algpseudocode}

\title{Truncated Inference for Latent Variable Optimization Problems: Application to Robust Estimation and Learning}

\author{%
  Christopher Zach\thanks{This work was partially supported by the Wallenberg AI, Autonomous Systems and Software Program (WASP) funded by the Knut and Alice Wallenberg Foundation.
} \qquad Huu Le
  \\
  Chalmers University of Technology\\
  Gothenburg, Sweden \\
  \texttt{\{zach,huul\}@chalmers.se}
}

\theoremstyle{plain}

\newtheorem{proposition}{Proposition}
\newtheorem{lemma}{Lemma}
\theoremstyle{definition}
\newtheorem{remark}{Remark}

\providecommand{\norm}[1]{\lVert#1\rVert}
\providecommand{\Norm}[1]{\left\lVert#1\right\rVert}

\def\v#1{\ensuremath{\mathbf{#1}}}

\def\vu{\v u}

\def\cF{\ensuremath{\mathcal{F}}}

\providecommand{\lb}[1]{\underline{#1}}
\providecommand{\ub}[1]{\overline{#1}}

\providecommand{\EE}[2]{\ensuremath{\mathbb{E}_{#1}\left[ #2 \right]}}

\begin{document}

\maketitle

\begin{abstract}
  Optimization problems with an auxiliary latent variable structure in
  addition to the main model parameters occur frequently in computer vision
  and machine learning. The additional latent variables make the underlying
  optimization task expensive, either in terms of memory (by maintaining the
  latent variables), or in terms of runtime (repeated exact inference of
  latent variables). We aim to remove the need to maintain the latent
  variables and propose two formally justified methods, that dynamically adapt
  the required accuracy of latent variable inference. These methods have
  applications in large scale robust estimation and in learning energy-based
  models from labeled data.

\end{abstract}

\section{Introduction}
\label{sec:intro}

In this work we are interested in optimization problems that involve
additional latent variables and therefore have the general form,
\begin{align}
  \min_\theta \min_{\ub{\vu}} \ub{J}(\theta, \ub{\vu}) =: \min_\theta J(\theta),
  \label{eq:initial_objective}
\end{align}
where $\theta$ are the main parameters of interest and $\ub{\vu}$ denote the
complete set of latent variables. By construction
$\ub{J}(\theta, \ub{\vu})$ is always an upper bound to the ``ideal'' objective
$J$. In typical computer vision and machine learning settings the objective
function in Eq.~\ref{eq:initial_objective} has a more explicit structure as
follows,
\begin{align}
  \ub{J}(\theta, \ub{\vu}) = \frac{1}{N} \sum\nolimits_{i=1}^N \ub{J}_i(\theta, \ub{u}_i),
  \label{eq:main_objective}
\end{align}
where the index $i$ ranges over e.g.\ training samples or over observed
measurements.  Each $\ub{u}_i$ corresponds to the inferred (optimized) latent
variable for each term, and $\ub{\vu}$ is the entire collection of latent
variables, i.e.\ $\ub{\vu}=(\ub{u}_1,\dotsc,\ub{u}_N)$.  Examples for this
problem class are models for (structured) prediction with latent
variables~\cite{felzenszwalb2009object,yu2009learning}, supervised learning of
energy-based models~\cite{movellan1991contrastive,xie2003equivalence} (in both
scenarios $N$ labeled training samples are provided), and robust estimation
using explicit confidence weights~\cite{geman1992constrained,zach2014robust}
(where $N$ corresponds to the number of sensor measurements).

\begin{figure}[tb]
  \centering
  \subfigure[Relaxed generalized MM]{\includegraphics[width=0.45\textwidth]{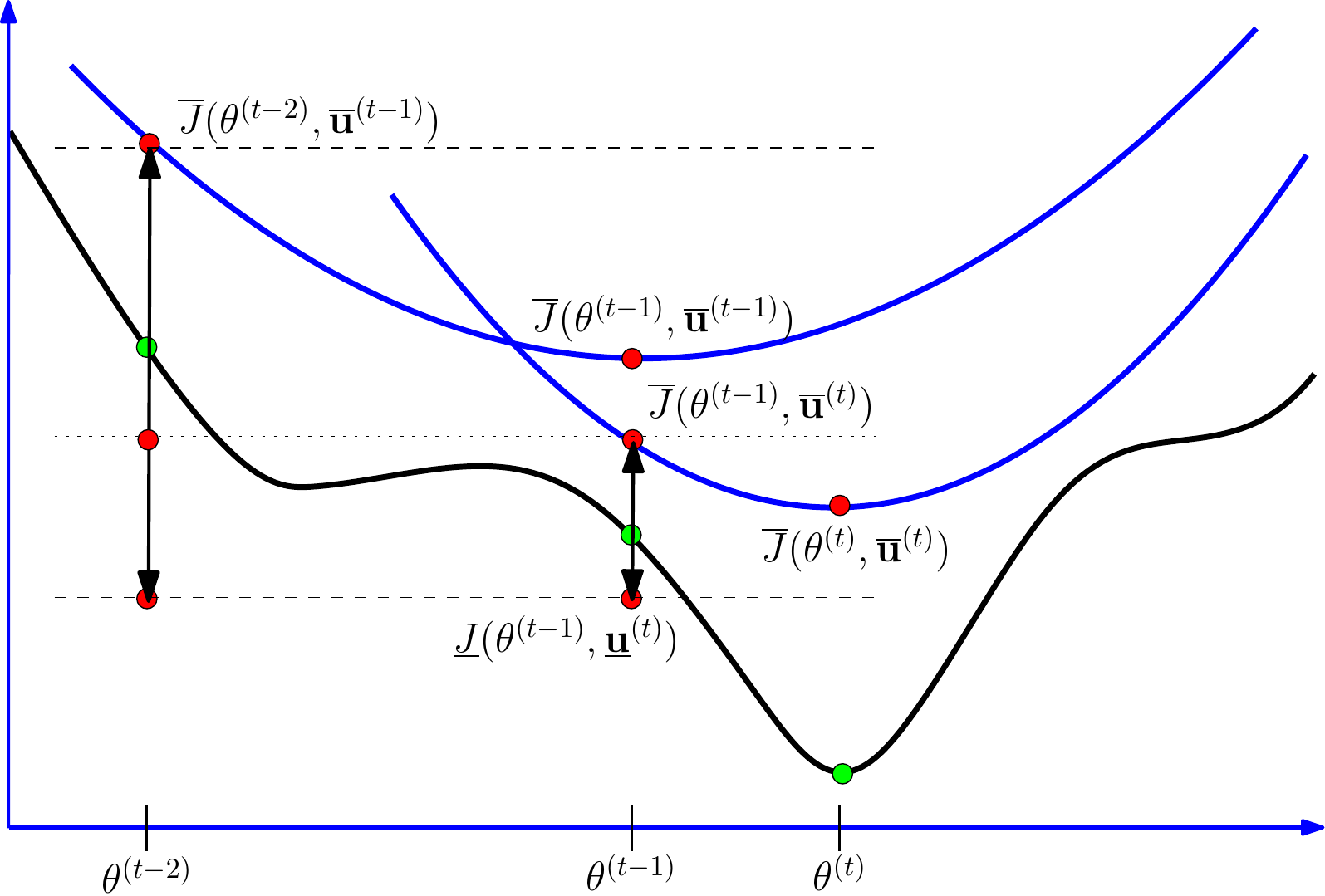}}~~
  \subfigure[Sufficient descent MM]{\includegraphics[width=0.45\textwidth]{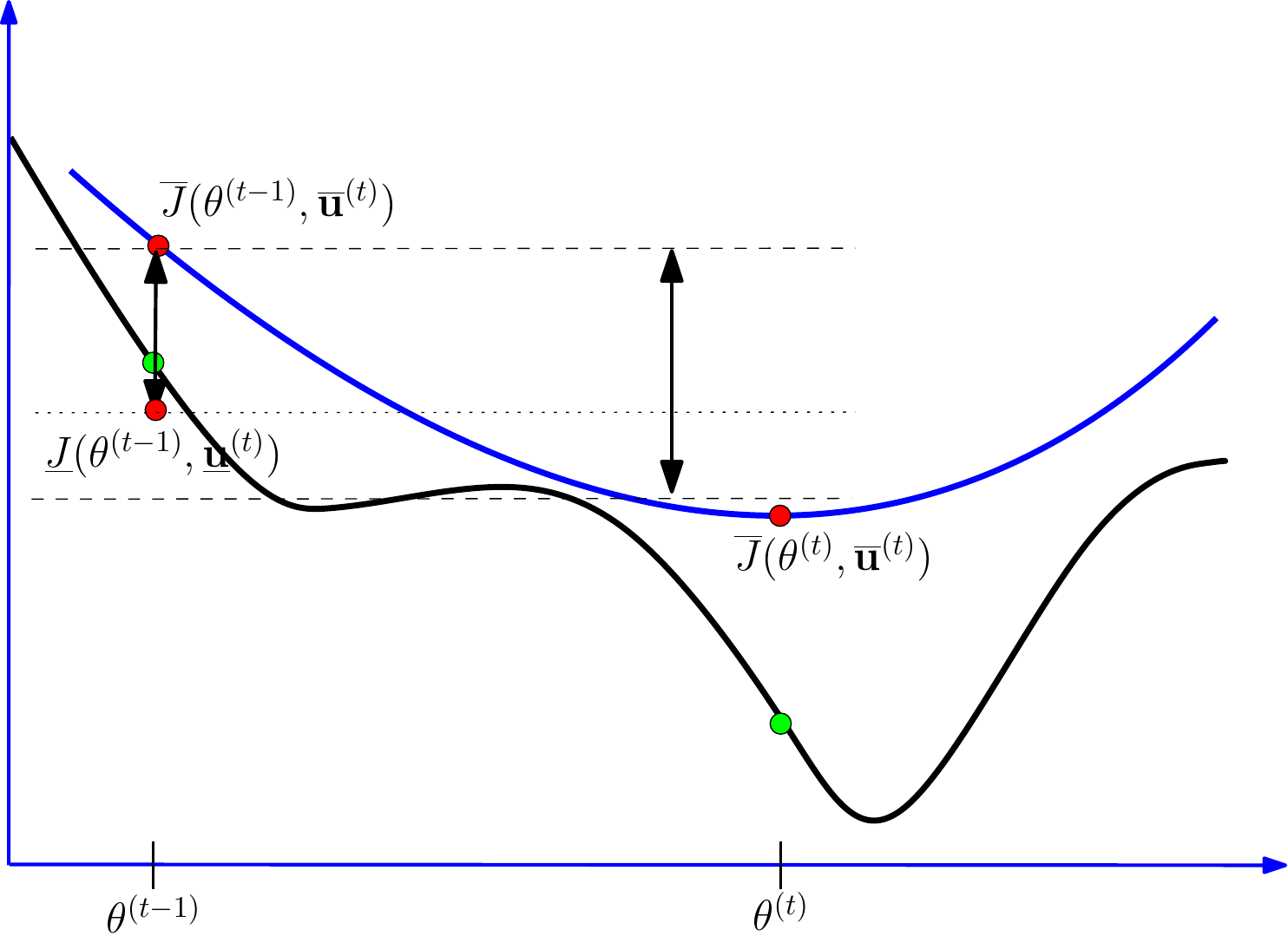}}
  \caption{Illustration of the principle behind our proposed
    majorization-minimization variants. Left: relaxed generalized MM requires
    that the current duality gap at $\theta^{(t-1)}$ (between dotted and lower
    dashed lines) is at most a given fraction of the gap induced by the
    previous upper bound (between dashed lines). Right: sufficient descent MM
    requires that the current duality gap (between upper dashed and dotted
    lines) is at most a given fraction of a guaranteed decrease (between
    dashed lines).  }
  \label{fig:teaser}
\end{figure}

We focus on the setting when $N$ is very large, and maintaining the values of
$\ub{u}_i$ for all $N$ terms in memory is intractable. In particular, storing
the entire vector $\ub{\vu}$ is undesirable when the dimensionality of each
$\ub{u}_i$ is large. In one of our applications $\ub{u}_i$ represents the
entire set of unit activations in a deep neural network, and therefore
$\ub{u}_i$ is high-dimensional in such cases.

Observe that neither $J(\theta)$ nor $\nabla J(\theta)$ are easy to evaluate
directly.  By using a variable projection approach, the loss $J$ in
Eq.~\ref{eq:main_objective} can in principle be optimized using a
``state-less'' gradient method,
\begin{align}
  \nabla J(\theta) = \frac{1}{N} \sum\nolimits_{i=1}^N \nabla_\theta \ub{J}_i(\theta; \ub{u}_i^*(\theta))
  \label{eq:gradient_step}
\end{align}
where
$\ub{u}_i^*(\theta)=\arg\min_{\ub{u}_i} \ub{J}_i(\theta; \ub{u}_i)$. Usually
determining $\ub{u}_i^*(\theta)$ requires itself an iterative minimization
method, hence exactly solving
$\arg\min_{\ub{u}_i} \ub{J}_i(\theta; \ub{u}_i)$ renders the computation of
$\nabla J(\theta)$ expensive in terms of run-time (e.g.\ it requires solving a
quadratic program in the application presented in Section~\ref{sec:CHL}). On
the other hand, by using Eq.~\ref{eq:gradient_step} there is no need to
explicitly keep track of the values $\ub{u}_i^*(\theta)$ (as long as
determining the minimizer $\ub{u}_i^*(\theta)$ is ``cold-started'', i.e.\ run
from scratch). Note that Eq.~\ref{eq:gradient_step} is only correct for
stationary points $\ub{u}_i^*(\theta)$. For inexact minimizers
$\ub{u}_i'(\theta) \approx \ub{u}_i^*(\theta)$
%(e.g.\ obtained by truncating the number of optimization steps)
the second term in the total derivative,
\begin{align}
  \frac{d \ub{J}_i(\theta; \ub{u}_i'(\theta))}{d\theta} = \frac{\partial \ub{J}_i(\theta; \ub{u}_i)}{\partial \theta} \Big|_{\ub{u}_i=\ub{u}_i'(\theta)}
  + \frac{\partial \ub{J}_i(\theta; \ub{u}_i)}{\partial \ub{u}_i}\Big|_{\ub{u}_i=\ub{u}_i'(\theta)} \cdot \frac{\partial \ub{u}_i'(\theta)}{\partial\theta}
\end{align}
does not vanish, and the often complicated dependence of $\ub{u}_i'(\theta)$
on $\theta$ must be explicitly modeled (e.g.\ by ``un-rolling'' the iterations
of a chosen minimization method yielding $\ub{u}_i'(\theta)$). Otherwise, the
estimate for $\nabla_\theta J$ will be biased, and minimization of $J$ will be
eventually hindered. Nevertheless, we are interested in such inexact solutions
$\ub{u}'_i(\theta)$, that can be obtained in finite time (without
warm-starting from a previous estimate), and the question is how close
$\ub{u}'_i(\theta)$ has to be to $\ub{u}^*_i(\theta)$ in order to still
successfully minimize Eq.~\ref{eq:main_objective}. Hence, we are interested in algorithms
that have the following properties:
\begin{enumerate}
\item returns a minimizer (or in general a stationary point) of Eq.~\ref{eq:main_objective},
\item does not require storing $\ub{\vu}=(\ub{u}_1,\dotsc,\ub{u}_N)$ between updates of $\theta$,
\item and is optionally applicable in a stochastic or incremental setting.
\end{enumerate}
We propose two algorithms to minimize Eq.~\ref{eq:main_objective}, that
leverage
% truncated and therefore
inexact minimization for the latent variables $\ub{u}_i$ (described in
Sections~\ref{sec:relaxed_MM} and~\ref{sec:grad_MM}). Our analysis applies to
the setting, when each $\ub{J}_i(\theta; \ub{u}_i)$ is convex in
$\ub{u}_i$. The basic principle is illustrated in Fig.~\ref{fig:teaser}: in
iteration $t$ of each of the proposed algorithms, a new upper bound
parametrized by $\ub{u}^{(t)}$ is found, that guarantees a sufficient
improvement over the previous upper bound according to a respective
criterion. This criterion either uses past objective values
(Fig.~\ref{fig:teaser}(a)) or current gradient information
(Fig.~\ref{fig:teaser}(b)). In Section~\ref{sec:results} we demonstrate the
proposed algorithms for large scale robust estimation instances and for
training a layered energy-based model.

\section{Related Work}
\label{sec:related}

Our proposed methods are based on the majorization-minimization (MM)
principle~\cite{lange2000optimization,hunter2004tutorial}, which generalizes
methods such as
expectation-maximization~\cite{dempster1977maximum,wu1983convergence,neal1998view}
and the convex-concave procedure~\cite{yuille2003concave}. A large number of
variants and extensions of MM exist. The notion of a (global) majorizer is
relaxed in~\cite{mairal2013optimization,mairal2015incremental}, where also a
stochastic variant termed MISO (Minimization by Incremental Surrogate
Optimization) is proposed. The memory consumption of MISO is $O(ND)$, as
sufficient information about each term in Eq.~\ref{eq:main_objective} has to
be maintained. Here $D$ is the size of the data necessary to represent a
surrogate function (i.e.\ $D=\dim(\ub{u}_i)$).
The first-order surrogates introduced in~\cite{mairal2013optimization} are
required to agree with the gradient at the current solution, which is relaxed
to asymptotic agreement in~\cite{xu2016relaxed}.

The first of our proposed methods is based on the ``generalized MM'' method
presented in~\cite{parizi2019generalized}, which relaxes the ``touching
condition'' in MM by a looser diminishing gap criterion. Our second method is
also a variant of MM, but it is stated such that it easily transfers to a
stochastic optimization setting. Since our surrogate functions are only upper
bounds of the true objective, the gradient induced by a mini-batch will be
biased even at the current solution. This is different from
e.g.~\cite{zhang2019generalized}, where noisy surrogate functions are
considered, which have unbiased function values and gradients at the current
solution. The \emph{stochastic successive upper-bound minimization} (SSUM)
algorithm~\cite{razaviyayn2016stochastic} averages information from the
surrogate functions gathered during the iterations. Thus, for Lipschitz
gradient (quadratic) surrogates, the memory requirements reduce to
$O(D)$ (compared to $O(ND)$ for MISO).

Several gradient-based methods that are able to cope with noisy gradient
oracles are presented
in~\cite{bertsekas2000gradient,devolder2014first,dvurechensky2016stochastic}
with different assumptions on the objective function and on the gradient
oracle,

Majorization-minimization is strongly connected to minimization by alternation
(AM). In~\cite{Csiszr1984InformationGA} a ``5-point'' property is proposed,
that is a sufficient condition for AM to converge to a global
minimum. Byrne~\cite{byrne2013alternating} points out that AM (and therefore
MM) fall into a larger class of algorithms termed ``sequential unconstrained
minimization algorithm'' (SUMMA).

Contrastive losses such as the one employed in Section~\ref{sec:CHL} occur
often when model parameters of latent variable models are estimated from
training data (e.g.~\cite{movellan1991contrastive,yu2009learning}). Such
losses can be interpreted either as finite-difference approximations to
implicit
differentiation~\cite{xie2003equivalence,scellier2017equilibrium,zach2019contrastive},
as surrogates for the misclassification loss~\cite{yu2009learning}, or as
approximations to the cross-entropy loss~\cite{zach2019contrastive}. Thus,
contrastive losses are an alternative to the exact gradient computation in
bilevel optimization problems (e.g.\ using the Pineda-Almeida
method~\cite{pineda1987generalization,almeida1990learning,scarselli2008graph}).

\section{Minimization Using Families of Upper Bounds}
\label{sec:methods}

%\subsection{Setting}
\paragraph{General setting}
Let $J: \mathbb{R}^d \to \mathbb{R}_{\ge 0}$ be a differentiable objective
function, that is bounded from below (we choose w.l.o.g.\
$J(\theta)\ge 0$ for all $\theta$). The task is to determine a minimizer
$\theta^*$ of $J$ (or stationary point in general).\footnote{By convergence to
  a stationary point we mean that the gradient converges to 0. Convergence of
  solution is difficult to obtain in the general non-convex setting.}  We
assume that $J$ is difficult to evaluate directly (e.g.\ $J$ has the form of
Eq.~\ref{eq:main_objective}), but a differentiable function
$\ub{J}(\theta; \ub{u})$ taking an additional argument
$\ub{u} \in \mathcal{U} \subseteq \mathbb{R}^{\bar D}$ is available that has
the following properties:
\begin{enumerate}
\item $\ub{J}(\theta,\ub{\vu}) \ge J(\theta)$ for all $\theta\in\mathbb{R}^d$ and $\ub{\vu}\in \mathcal{U}$,
\item $\ub{J}(\theta,\ub{\vu})$ is convex in $\ub{\vu}$ and satisfies strong duality,
\item $J(\theta) = \min_{\ub{\vu} \in \mathcal{U}} \ub{J}(\theta,\ub{\vu})$.
\end{enumerate}
This means that $\ub{J}(\theta,\ub{\vu})$ is a family of upper bounds of $J$
parametrized by $\ub{\vu}\in\mathcal{U}$, and the target objective $J(\theta)$
is given as the lower envelope of $\{\ub{J}(\theta,\ub{\vu}) : \ub{\vu}
\in\mathcal{U}\}$. The second condition implies that optimizing the upper
bound for a given $\theta$ is relatively easy (but in general it still will
require an iterative algorithm). As pointed out in Section~\ref{sec:intro},
$\ub{\vu}$ may be very high-dimensional and expensive to maintain in memory.
We will absorb the constraint $\ub{u}\in\mathcal{U}$ into $\ub{J}$ and
therefore drop this condition in the following.

\paragraph{The baseline algorithm: minimization by alternation}
The straightforward method to minimize $J$ in
Eq.~\ref{eq:initial_objective}/Eq.~\ref{eq:main_objective} is by alternating
minimization (AM) w.r.t.\ $\theta$ and $\ub{u}$.
The downside of AM is, that the entire set of latent variables represented by
$\ub{\vu}$ has to be stored while updating $\theta$. This can be intractable
in machine learning applications when $N\gg 1$ and $D \gg 1$.

\section{Relaxed Generalized Majorization-Minimization}
\label{sec:relaxed_MM}

Our first proposed method extends the generalized
majorization-minimization
method~\cite{parizi2019generalized} to the case when computation of
$J$ is expensive. Majorization-minimzation
(MM,~\cite{lange2000optimization,hunter2004tutorial}) maintains a sequence of
solutions $(\theta^{(t)})_{t=1}^T$ and latent variables
$(\ub{\vu}^{(t)})_{t=1}^T$ such that
\begin{align}
  \theta^{(t-1)} &\gets \arg\min_\theta \ub{J}(\theta, \ub{\vu}^{(t-1)}) & \ub{u}^{(t)} &\gets \arg\min_{\ub{\vu}} \ub{J}(\theta^{(t-1)}, \ub{u}).
\end{align}
Standard MM requires the following ``touching condition'' to be satisfied,
\begin{align}
  \ub{J}(\theta^{(t-1)}, \ub{\vu}^{(t)}) = J(\theta^{(t-1)}).
\end{align}
It should be clear that a standard MM approach is equivalent to the
alternating minimization baseline algorithm. In most applications of MM, the
domain of the latent variables defining the upper bound is identical to the
domain for $\theta$.

Generalized MM relaxes the touching condition to the following one,
\begin{align}
  \ub{J}(\theta^{(t-1)},& \ub{\vu}^{(t)}) \le \eta J(\theta^{(t-1)}) + (1-\eta) \ub{J}(\theta^{(t-1)}, \ub{\vu}^{(t-1)}) \nonumber \\
  {} &= \ub{J}(\theta^{(t-1)}, \ub{\vu}^{(t-1)})- \eta \underbrace{\left( \ub{J}(\theta^{(t-1)}, \ub{\vu}^{(t-1)}) - J(\theta^{(t-1)}) \right)}_{=: d_t},
\end{align}
where $\eta\in(0,1)$ is a user-specified parameter. By construction the gap
$d_t=\ub{J}(\theta^{(t-1)}, \ub{\vu}^{(t-1)}) - J(\theta^{(t-1)})$ is
non-negative. The above condition means that $\ub{\vu}^{(t)}$ has to be chosen
such that the new objective value $\ub{J}(\theta^{(t)}, \ub{\vu}^{(t)})$
is guaranteed to sufficiently improve over the current upper bound $\ub{J}(\theta^{(t-1)}, \ub{\vu}^{(t-1)})$,
\begin{align}
  \ub{J}(\theta^{(t)}, \ub{\vu}^{(t)}) &\le \ub{J}(\theta^{(t-1)}, \ub{\vu}^{(t)})
  \le \ub{J}(\theta^{(t-1)}, \ub{\vu}^{(t-1)}) - \eta d_t. \nonumber
\end{align}
It is shown that the sequence $\lim_{t\to\infty} d_t\to 0$, i.e.\
asymptotically the true cost $J$ is optimized. Since generalized MM decreases
the upper bound less aggressively than standard MM, it has an improved
empirical ability to reach better local minima in highly non-convex
problems~\cite{parizi2019generalized}.

Generalized MM is not directly applicable in our setting, as $J$ is assumed
not to be available (or at least expensive to compute, which is exactly we aim
to avoid). By leveraging convex duality we have a lower bound for
$\lb{J}(\theta, \lb{\vu}) \le J(\theta)$ available. Hence, we modify the
generalized MM approach by replacing $J(\theta^{(t-1)})$ with a lower bound
$\lb{J}(\theta^{(t-1)},\lb{\vu}^{(t)})$ for a suitable dual parameter
$\lb{\vu}^{(t)}$, leading to a condition on $\ub{\vu}^{(t)}$ and $\lb{\vu}^{(t)}$ of the form
\begin{align}
  \ub{J}(\theta^{(t-1)}, \ub{\vu}^{(t)}) \le \eta \lb{J}(\theta^{(t-1)}, \lb{\vu}^{(t)}) %\nonumber \\
                                       + (1-\eta) \ub{J}(\theta^{(t-1)}, \ub{\vu}^{(t-1)}). \nonumber
\end{align}
This condition still has the significant shortcoming, that both
$\ub{J}(\theta^{(t)}, \ub{\vu}^{(t-1)})$ and
$\ub{J}(\theta^{(t-1)}, \ub{\vu}^{(t-1)})$ need to be evaluated. While
computation of the first quantity is firmly required, evaluation of the second
value is unnecessary as we will see in the following. Not needing to compute
$\ub{J}(\theta^{(t-1)}, \ub{\vu}^{(t-1)})$ also means that the memory associated
with $\ub{\vu}^{(t-1)}$ can be immediately reused.
Our proposed condition on $\ub{\vu}^{(t)}$ and $\lb{\vu}^{(t)}$ for a \emph{relaxed generalized} MM (or \emph{ReGeMM})
method is given by
\begin{align}
  \ub{J}(\theta^{(t-1)}, \ub{\vu}^{(t)}) &\le \eta \lb{J}(\theta^{(t-1)},\lb{\vu}^{(t)}) + (1-\eta) \ub{J}(\theta^{(t-2)}, \ub{\vu}^{(t-1)}),
  \label{eq:relaxed_MM}
\end{align}
where $\eta \in (0,1)$, e.g.\ $\eta=1/2$ in our implementation. The resulting
algorithm is given in Alg.~\ref{alg:relaxed_MM}.

\begin{algorithm}[htb]
  \begin{algorithmic}[1]
    \Require Initial $\theta^{(0)}=\theta^{(-1)}$ and $\ub{\vu}^{(0)}$, number of rounds $T$
    \For{$t=1,\dotsc,T$}
    \State Determine $\ub{\vu}^{(t)}$ and $\lb{\vu}^{(t)}$ that satisfy Eq.~\ref{eq:relaxed_MM}
    \State Set $\theta^{(t)} \gets \arg\min_\theta \ub{J}(\theta, \ub{\vu}^{(t)})$
    \EndFor
    \State \Return $\theta^{(T)}$
  \end{algorithmic}
  \caption{ReGeMM: Relaxed Generalized Majorization-Minimization}
  \label{alg:relaxed_MM}
\end{algorithm}
The existence of a pair $(\ub{\vu}^{(t)},\lb{\vu}^{(t)})$ is guaranteed, since
both $\lb{J}(\theta^{(t-1)}; \lb{\vu}^{(t)})$ and
$\ub{J}(\theta^{(t-1)}; \ub{\vu}^{(t)})$ can be made arbitrarily close to
$J(\theta^{(t-1)})$ by our assumption of strong duality.
We introduce $c_t$,
\begin{align}
  c_t := \ub{J}(\theta^{(t-2)},\ub{\vu}^{(t-1)}) - \lb{J}(\theta^{(t-1)},\lb{\vu}^{(t)}) \ge 0,
\end{align}
and Eq.~\ref{eq:relaxed_MM} can therefore be restated as
\begin{align}
  \ub{J}(\theta^{(t-1)}, \ub{\vu}^{(t)}) \le \ub{J}(\theta^{(t-2)},\ub{\vu}^{(t-1)}) - \eta c_t.
\end{align}
\begin{proposition}
  We have $\lim_{t\to\infty} c_t = 0$.
\end{proposition}
\begin{proof}
  We define $v_t := \ub{J}(\theta^{(t-2)},\ub{\vu}^{(t-1)}) - \eta c_t$.
  First, observe that
  \begin{align}
    c_t &= \ub{J}(\theta^{(t-2)}; \ub{\vu}^{(t-1)}) - \lb{J}(\theta^{(t-1)}; \lb{\vu}^{(t)})
    \ge \ub{J}(\theta^{(t-1)}; \ub{\vu}^{(t-1)}) - \lb{J}(\theta^{(t-1)}; \lb{\vu}^{(t)}) \nonumber\\
    &\ge \ub{J}(\theta^{(t-1)}; \ub{\vu}^{(t-1)}) - J(\theta^{(t-1)}) \ge 0 \nonumber
  \end{align}
  (using the relations $\ub{J}(\theta^{(t-1)}; \ub{\vu}^{(t-1)}) \le
  \ub{J}(\theta^{(t-2)}; \ub{\vu}^{(t-1)})$ and $\lb{J}(\theta^{(t-1)}; \lb{\vu})
  \le J(\theta^{(t-1)}) $ $\le \ub{J}(\theta^{(t-1)}; \ub{\vu})$ for any $\lb{\vu}$
  and $\ub{\vu}$). We further have
  \begin{align}
    \sum\nolimits_{t=1}^T c_t &= \frac{1}{\eta} \sum\nolimits_{t=1}^T \left( \ub{J}(\theta^{(t-2)}; \ub{\vu}^{(t-1)}) - v_t \right) \nonumber \\
    &\le \frac{1}{\eta} \sum\nolimits_{t=1}^T \left( \ub{J}(\theta^{(t-2)}; \ub{\vu}^{(t-1)}) - \ub{J}(\theta^{(t-1)}; \ub{\vu}^{(t)}) \right) \nonumber \\
    &= \frac{1}{\eta} \left( \ub{J}(\theta^{(-1)}; \ub{\vu}^{(0)}) - \ub{J}(\theta^{(T-1)}; \ub{\vu}^{(T)}) \right) < \infty, \nonumber
  \end{align}
  since $\ub{J}$ is bounded from below. In the first line we used the
  definition of $d_t$ and in the second line we utilized that
  $\ub{J}(\theta^{(t-1)}; \ub{\vu}^{(t)})\le v_t$. The last line follows from
  the telescopic sum. Overall, we have that
  \begin{align}
    \lim_{T\to\infty} \sum\nolimits_{t=1}^T c_t &= \frac{\ub{J}(\theta^{(-1)}; \ub{\vu}^{(0)}) - \lim_{T\to\infty} \ub{J}(\theta^{(T-1)}; \ub{\vu}^{(T)})}{\eta} \nonumber, % < \infty
  \end{align}
  which is finite, since $J$ (and therefore $\ub{J}$) is bounded from
  below. From $c_t \ge 0$ and
  $\lim_{T\to\infty} \sum_{t=1}^T c_t < \infty$ we deduce that
  $\lim_{t\to\infty} c_t=0$.
\end{proof}
Hence, in analogy with the generalized MM method~\cite{parizi2019generalized},
the upper bound $\ub{J}(\theta^{(t)},\ub{\vu}^{(t)})$ approaches the target
objective value $J(\theta^{(t)})$ in the proposed relaxed scheme.
This result also implies that finding $\ub{\vu}^{(t)}$ will be increasingly
harder. This is expected, since one ultimately aims to minimize $J$. If we
additionally assume that the mapping
$\theta\mapsto \nabla_\theta\ub{J}(\theta,\ub{\vu})$ has Lipschitz gradient
for all $\ub{\vu}$, then it can be also shown that
$\nabla_\theta\ub{J}(\theta^{(t)},\ub{\vu}^{(t)}) \to 0$ (we refer to the
appendix).

The relaxed generalized MM approach is therefore a well-understood method when
applied in a full batch scenario (recall Eq.~\ref{eq:main_objective}). Since
the condition in Eq.~\ref{eq:relaxed_MM} is based on all terms in the
objectives, it is not clear how it generalizes to an incremental or stochastic
setting, when $\theta$ is updated using small mini-batches. This is the
motivation for developing an alternative criterion to Eq.~\ref{eq:relaxed_MM}
in the next section, that is based on ``local'' quantities.

%\subsection{Using constant memory}
\paragraph{Using constant memory}
\label{sec:constant_memory}

Naive implementations of Alg.~\ref{alg:relaxed_MM} require $O(N)$ memory to
store $\ub{\vu}=(\ub{u}_1,\dotsc,\ub{u}_N)$. In many applications the number
of terms $N$ is large, but the latent variables $(\ub{u}_i)_i$ have the same
structure for all $i$ (e.g.\ $\ub{u}_i$ represent pixel-level predictions for
training images of the same dimensions). If we use a gradient method to update
$\theta$, then the required quantities can be accumulated in-place, as shown
in the appendix.
The constant memory algorithm is not limited to first order methods for
$\theta$, but any method that accumulates the information needed to determine
$\theta^{(t)}$ from $\theta^{(t-1)}$ in-place is feasible (such as the Newton
or the Gauss-Newton method).

\section{Sufficient Descent Majorization-Minimization}
\label{sec:grad_MM}

The ReGeMM method proposed above has two disadvantages: (i) the underlying
condition is somewhat technical and it is also a global condition, and (ii)
the resulting algorithm does not straightforwardly generalize to incremental
or stochastic methods, that have proven to be far superior compared to
full-batch approaches, especially in machine learning scenarios.

In this section we make the additional assumption on $\ub{J}$, that
\begin{align}
  \ub{J}(\theta', \ub{\vu}) &\le \ub{J}(\theta, \ub{\vu})
  + \nabla_\theta \ub{J}(\theta, \ub{\vu})^T (\theta'-\theta) + \frac{L}{2} \norm{\theta'-\theta}^2,
  \label{eq:Lipschitz_gradient}
\end{align}
for a constant $L>0$ and all $\ub{\vu}$. This essentially means, that the
mapping $\theta \mapsto \ub{J}(\theta, \ub{\vu})$ has a Lipschitz gradient
with Lipschitz constant $L$. This assumption is frequent in many
gradient-based minimization methods. Note that the minimizer of the r.h.s.\ in
Eq.~\ref{eq:Lipschitz_gradient} w.r.t.\ $\theta'$ is given by $\theta' =
\theta - \frac{1}{L} \nabla_\theta \ub{J}(\theta, \ub{\vu})$.
Hence, we focus on gradient-based updates of $\theta$ in the following, i.e.\
$\theta^{(t)}$ is given by
\begin{align}
  \theta^{(t)} = \theta^{(t-1)} - \frac{1}{L} \nabla_\theta \ub{J}(\theta^{(t-1)}, \ub{\vu}^{(t)}).
\end{align}
Combining this with Eq.~\ref{eq:Lipschitz_gradient} yields
\begin{align}
  \ub{J}(\theta^{(t)}, \ub{\vu}^{(t)}) &\le \ub{J}(\theta^{(t-1)}, \ub{\vu}^{(t)})
  - \frac{1}{2L} \norm{\nabla_\theta \ub{J}(\theta^{(t-1)}, \ub{\vu}^{(t)})}^2, \nonumber
\end{align}
hence the update from $\theta^{(t-1)}$ to $\theta^{(t)}$ yields a guaranteed
reduction of $\ub{J}(\cdot, \ub{\vu}^{(t)})$ in terms of the respective gradient
magnitude.

We therefore propose the following condition on
$(\ub{\vu}^{(t)},\lb{\vu}^{(t)})$ based on the current iterate
$\theta^{(t-1)}$: for a $\rho\in(0,1)$ (which is set to $\rho=1/2$ in our
implementation) determine $\ub{\vu}^{(t)}$ and $\lb{\vu}^{(t)}$ such that
\begin{align}
  \ub{J}(\theta^{(t-1)}; \ub{\vu}^{(t)}) - \lb{J}(\theta^{(t-1)}; \lb{\vu}^{(t)}) \le \frac{\rho}{2L} \norm{\nabla \ub{J}(\theta^{(t-1)}; \ub{\vu}^{(t)})}^2
  \label{eq:gradient_condition}
\end{align}
This condition requires intuitively, that the duality gap
$\ub{J}(\theta^{(t-1)}; \ub{\vu}^{(t)}) - \lb{J}(\theta^{(t-1)}; \lb{\vu}^{(t)})$
is sufficiently smaller than the reduction of $\ub{J}(\cdot;
\ub{u}^{(t)})$  guaranteed by a gradient descent step.
Rearranging the above condition (and using that
$\theta^{(t)}=\theta^{(t-1)}- \nabla\ub{J}(\theta^{(t-1)},\ub{\vu}^{(t)})/L$)
yields
\begin{align}
  \ub{J}(\theta^{(t)}; \ub{\vu}^{(t)}) &\le \ub{J}(\theta^{(t-1)}; \ub{\vu}^{(t)}) - \frac{1}{2L} \norm{\nabla \ub{J}(\theta^{(t-1)}; \ub{\vu}^{(t)})}^2 \nonumber \\
  {} &\le \lb{J}(\theta^{(t-1)}; \lb{\vu}^{(t)}) - \frac{1-\rho}{2L} \norm{\nabla \ub{J}(\theta^{(t-1)}; \ub{\vu}^{(t)})}^2 \nonumber \\
  {} &\le J(\theta^{(t-1)}) - \frac{1-\rho}{2L} \norm{\nabla \ub{J}(\theta^{(t-1)}; \ub{\vu}^{(t)})}^2 \nonumber,
\end{align}
i.e.\ the upper bound at the new solution $\theta^{(t)}$ is sufficiently below
the lower bound (and the true function value) at the current solution
$\theta^{(t-1)}$. This can be stated compactly,
\begin{align}
  J(\theta^{(t)}) &\le \ub{J}(\theta^{(t)}; \ub{\vu}^{(t)})
                    \le J(\theta^{(t-1)}) - \frac{1-\rho}{2L} \norm{\nabla \ub{J}(\theta^{(t-1)}; \ub{\vu}^{(t)})}^2,
  \label{eq:reduction}
\end{align}
and the sequence $(J(\theta^{(t)})_{t=1}^{\infty}$ is therefore
non-increasing. Since we are always asking for a sufficient decrease (in
analogy with the Armijo condition), we expect convergence to a stationary
solution $\theta^*$. This is the case:
\begin{proposition}
  $\lim_{t\to\infty} \nabla_\theta \ub{J}(\theta^{(t-1)}; \ub{\vu}^{(t)}) = 0$.
\end{proposition}
\begin{proof}
  By rearranging Eq.~\ref{eq:reduction} we have
  \begin{align}
    \sum\nolimits_t &\norm{\nabla \ub{J}(\theta^{(t-1)}; \ub{\vu}^{(t)})}^2
             \le \frac{2L}{1-\rho} \sum\nolimits_t \left( J(\theta^{(t-1)}) - J(\theta^{(t)}) \right)
                      = \frac{2L}{1-\rho} \left( J(\theta^{(0)}) - J(\theta^*) \right) < \infty \nonumber,
  \end{align}
  and therefore $\norm{\nabla \ub{J}(\theta^{(t-1)}; \ub{\vu}^{(t)})} \to 0$,
  which implies that $\nabla \ub{J}(\theta^{(t-1)}; \ub{\vu}^{(t)}) \to 0$.
\end{proof}
We summarize the resulting \emph{sufficient descent MM} (or \emph{SuDeMM})
method in Alg.~\ref{alg:SuDeMM}.
As with the ReGeMM approach, determining $\ub{u}$ is more difficult when
closing in on a stationary point (as
$\nabla_\theta \ub{J}(\theta^{(t)},\ub{u}) \to 0$). The gradient step
indicated in line~3 in Alg.~\ref{alg:SuDeMM} can be replaced by any
update that guarantees sufficient descent. Finally, in analogy with the ReGeMM
approach discussed in the previous section, it is straightforward to obtain a
constant memory variant of Alg.~\ref{alg:SuDeMM}. The stochastic method
described below incorporates both immediate memory reduction from
$O(N)$ to $O(B)$, where $B$ is the size of the mini-batch, and faster
minimization due to the use of mini-batches.

\begin{algorithm}[tb]
  \begin{algorithmic}[1]
    \Require Initial $\theta^{(0)}$, number of rounds $T$
    \For{$t=1,\dotsc,T$}
    \State Determine $\ub{\vu}^{(t)}$ and $\lb{\vu}^{(t)}$ that satisfy Eq.~\ref{eq:gradient_condition}
    \State Set $\theta^{(t)} \gets \theta^{(t-1)} - \frac{1}{L} \nabla_\theta \ub{J}(\theta^{(t-1)}, \ub{\vu}^{(t)})$
    \EndFor
    \State \Return $\theta^{(T)}$
  \end{algorithmic}
  \caption{SuDeMM: Sufficient-Descent Majorization-Minimization}
  \label{alg:SuDeMM}
\end{algorithm}

\subsection{Extension to the stochastic setting}

In many machine learning applications $J$ will be of the form of
Eq.~\ref{eq:main_objective} with $N\gg 1$ being the number of training
samples. It is well known that in such settings methods levering the full
gradient accumulated over all training samples are hugely outperformed by
stochastic gradient methods, which operate on a single training sample (i.e.\
term in Eq.~\ref{eq:main_objective}) or, alternatively, on a small mini-batch
of size $B$ randomly drawn from the range $\{1,\dotsc,N\}$.

It is straightforward to extend Alg.~\ref{alg:SuDeMM} to a stochastic
setting working on single data points (or mini-batches) by replacing the
objective values $\ub{J}(\theta^{(t-1)}; \ub{\vu}^{(t)})$,
$\lb{J}(\theta^{(t-1)}; \lb{u}^{(t)})$) and the full gradient
$\nabla_\theta \ub{J}(\theta^{(t-1)}, \ub{\vu}^{(t)})$ with the respective
mini-batch counter-parts. The resulting algorithm is depicted in
Alg.~\ref{alg:SuDeMM_stochastic} (for mini-batches of size~one). Due to the
stochastic nature of the gradient estimate
$\nabla_\theta \ub{J}_i(\theta^{(t-1)}, \ub{\vu}^{(t)})$, both the step sizes
$\alpha_t>0$ and the reduction parameter $\rho_t>0$ are time-dependent and
need to satisfy the following conditions,
\begin{align}
  \sum\nolimits_{t=1}^\infty \alpha_t = \infty & & \sum\nolimits_{t=1}^\infty \alpha_t^2 < \infty & & \sum\nolimits_{t=1}^\infty \rho_t < \infty.
\end{align}
The first two conditions on the step sizes $(\alpha_t)_t$ are standard in
stochastic gradient methods, and the last condition on the sequence
$(\rho_t)_t$ ensures that the added noise by using time-dependent upper bounds
$\ub{J}(\cdot,\ub{\vu}^{(t)})$ (instead of the time-independent function
$J(\cdot)$) has bounded variance. The constraint on $\rho_t$ is therefore stronger
than the intuitively necessary condition
$\rho_t\stackrel{t\to\infty}\to 0$. We refer to the appendix for
a detailed discussion. Due to the small size $B$ of a mini-batch, the values
of $\ub{\vu}^{(t)}_i$ and $\lb{\vu}^{(t)}_i$ in the mini-batch can be
maintained, and the restarting strategy outlined in
Section~\ref{sec:constant_memory} is not necessary.

\begin{algorithm}[tb]
  \begin{algorithmic}[1]
    \Require Initial $\theta^{(0)}$, number of rounds $T$
    \For{$t=1,\dotsc,T$}
    \State Uniformly sample $i$ from $\{1,\dotsc,N\}$
    \State Determine $\ub{u}_i$ and $\lb{u}_i$ that satisfy
    \begin{align}
      \ub{J}_i(\theta^{(t-1)},\ub{u}_i) - \lb{J}_i(\theta^{(t-1)},\lb{u}_i) \le \frac{\rho_t}{2} \norm{\nabla_\theta \ub{J}_i(\theta^{(t-1)},\ub{u}_i)}^2
    \end{align}
    \State Set $\theta^{(t)} \gets \theta^{(t-1)} - \alpha_t \nabla_\theta \ub{J}_i(\theta^{(t-1)}, \ub{u}^{(t)})$
    \EndFor
    \State \Return $\theta^{(T)}$
  \end{algorithmic}
  \caption{Stochastic Sufficient Descent Majorization-Minimization}
  \label{alg:SuDeMM_stochastic}
\end{algorithm}

\section{Applications}
\label{sec:results}

\subsection{Robust Bundle Adjustment}

In this experiment we first demonstrate the applicability of our proposed
ReGeMM schemes to a large scale robust fitting task. The aim is to determine
whether ReGeMM is also able to avoid poor local minima (in analogy with the
$k$-means experiment in~\cite{parizi2019generalized}). The hypothesis is, that
optimizing the latent variables just enough to meet the ReGeMM condition
(Eq.~\ref{eq:relaxed_MM}) corresponds to a particular variant of graduated
optimization, and therefore will (empirically) return better local minima for
highly non-convex problems.

Robust bundle adjustment aims to refine the camera poses and 3D point
structure to maximize a log-likelihood given image observations and
established correspondences. The unknowns are
$\theta=(P_1,\dotsc,P_n,X_1,\dotsc,X_m)$, where $P_k\in\mathbb{R}^6$ refers to
the $k$-th camera pose and $X_j\in\mathbb{R}^3$ is the position of the
$j$-th 3D point. The cost $J$ is given by
\begin{align}
  J(\theta) &= \sum\nolimits_i \psi(f_i(\theta) - m_i), \label{eq:BA_cost}
\end{align}
where $m_i\in\mathbb{R}^2$ is the $i$-th image observation and $f_i$ projects
the respective 3D point to the image plane of the corresponding camera.
$\psi$ is a so called robust kernel, which generally turns $J$ into a highly
non-convex objective functions with a large number of local
minima. Following~\cite{zach2014robust} an upper bound $\ub{J}$ is given via
half-quadratic (HQ) minimization~\cite{geman1992constrained},
\begin{align}
  \ub{J}(\theta, \ub{\vu}) &= \sum\nolimits_i \left( \frac{\ub{u}_i}{2} \norm{f_i(\theta) - m_i}^2 + \kappa(\ub{u}_i) \right),
\end{align}
where $\kappa: \mathbb{R}_{\ge 0} \to \mathbb{R}_{\ge 0}$ depends on the
choice for $\psi$, and $\ub{u}_i$ is identified as the (confidence) weight on
the $i$-th observation. The standard MM approach corresponds essentially to
the iteratively reweighted least squares method (IRLS), which is prone to
yield poor local minima if $\theta$ is not well initialized. For given
$\theta$ the optimal latent variables $\ub{u}_i^*(\theta)$ are given by
\begin{align}
  \ub{u}_i^*(\theta) = \omega(\norm{f_i(\theta) - m_i}),
\end{align}
where $\omega(\cdot)$ is the weight function associated with the robust kernel
$\psi$. Joint HQ minimization of $\ub{J}$ w.r.t.\ $\theta$ and $\ub{u}$ is
suggested and evaluated in~\cite{zach2014robust}, which empirically yields
significantly better local minima of $J$ than IRLS.
We compare this joint-HQ method (as well as IRLS and an explicit graduated
method~\cite{zach2018secrets}) with our ReGeMM condition
(Eq.~\ref{eq:relaxed_MM}), where the confidence weights $\ub{\vu}$ are
optimized to meet but not substantially surpass this criterion: a scale
parameter $\sigma\ge 1$ is determined such that $\ub{u}_i^{(t)}$ is set to
$ \omega(\norm{f_i(\theta^{(t-1)}) - m_i}/\sigma)$, and $\ub{\vu}$ satisfies
the ReGeMM condition (Eq.~\ref{eq:relaxed_MM}) and
\begin{align}
  \eta' J(\theta^{(t-1)}) + (1-\eta') \ub{J}(\theta^{(t-2)}, \ub{\vu}^{(t-1)}) &\le \ub{J}(\theta^{(t-1)}, \ub{\vu}^{(t)})
\end{align}
for an $\eta' \in (\eta, 1)$.
In our implementation we determine $\sigma$ using bisection search and choose
$\eta' = 3/4$. In this application the evaluation of $J$ is inexpensive, and
therefore we use $J(\theta^{(t-1)})$ instead of a lower bound
$\lb{J}(\theta^{(t-1)},\lb{\vu}^{(t)})$ in the r.h.s. The model parameters
$\theta$ are updated for given $\ub{\vu}$ using a Levenberg-Marquardt
solver. Our choice of $\psi$ is the smooth truncated quadratic
cost~\cite{zach2014robust}.

In Fig.~\ref{fig:results_BA_evolution} we depict the evolution of the target
objective Eq.~\ref{eq:BA_cost} for two metric bundle adjustment instances
from~\cite{agarwal2010bundle}. The proposed ReGeMM approach (with the initial
confidence weights $\ub{\vu}$ all set to~1) compares favorably against IRLS,
joint HQ~\cite{zach2014robust} and even graduated
optimization~\cite{zach2018secrets} (that leads only to a slightly better
minimum).%\footnote{Data provided by the authors of~\cite{zach2018secrets}.}
This observation is supported by comparing the methods using a larger database
of 20~problem instances~\cite{agarwal2010bundle} in
Fig.~\ref{fig:results_BA_final}, where the final objective values reached by
different methods are depicted. ReGeMM is again highly competitive. In
terms of run-time, ReGeMM is beetwen 5\% and 25\% slower than IRLS in our
implementation.

\begin{figure}[t]
  \centering
  \includegraphics[width=0.48\textwidth]{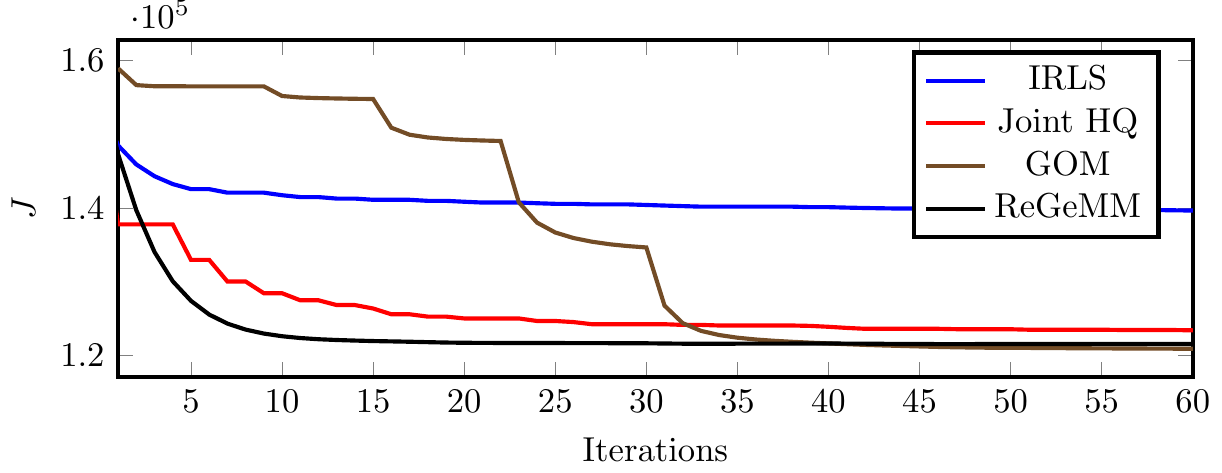}
  \includegraphics[width=0.48\textwidth]{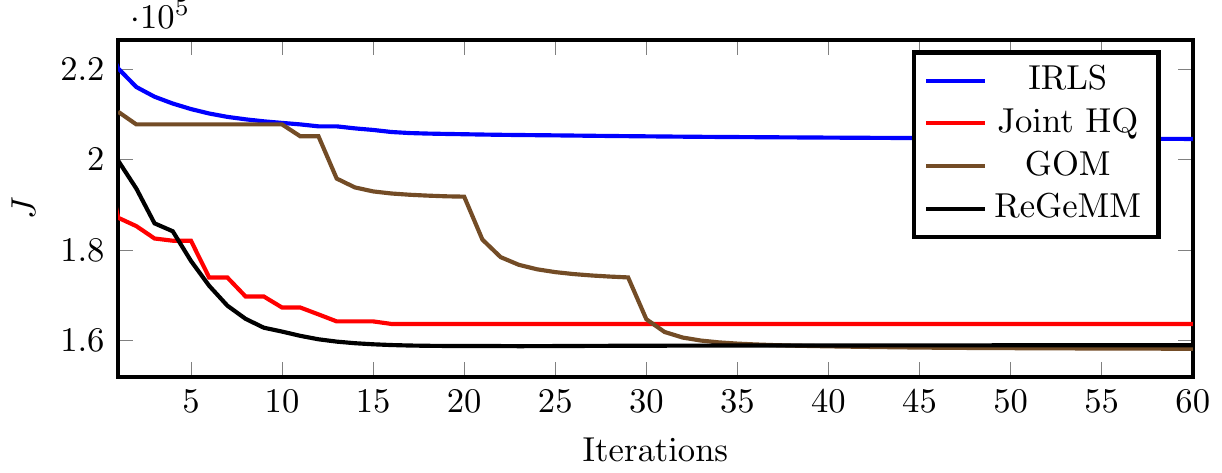}
  \caption{Objective value w.r.t.\ number of iterations of a NNLS solver
    for the Dubrovnik-356 (left) and Venice-427 (right) datasets. }
  \label{fig:results_BA_evolution}
\end{figure}

\begin{figure}[t]
  \centering
  \includegraphics[width=0.98\textwidth]{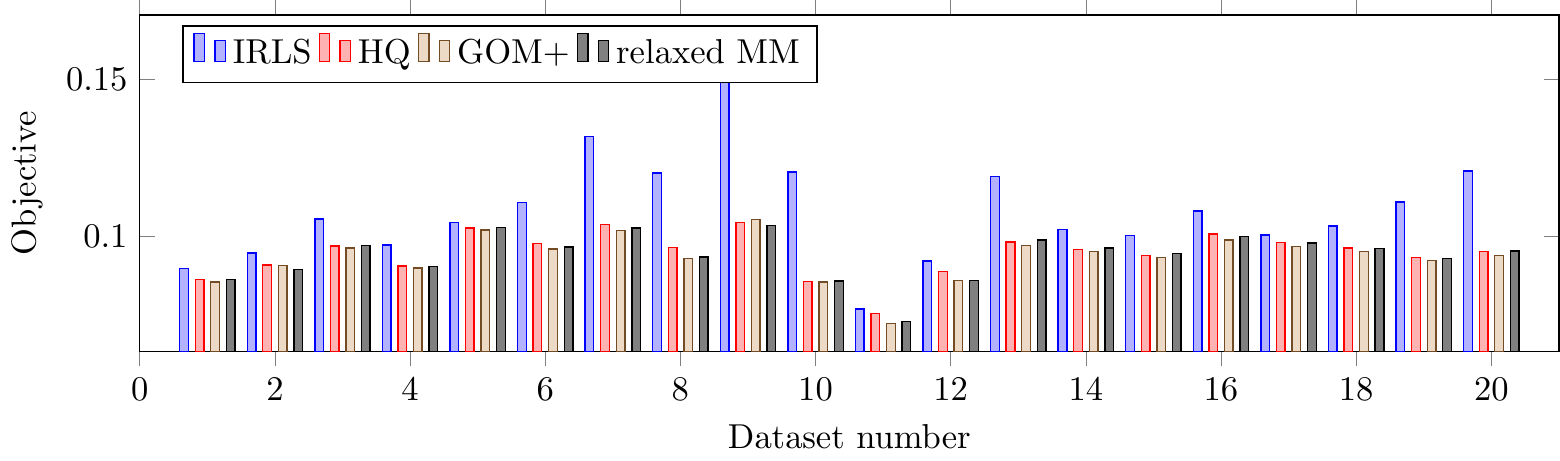}
  \caption{Final objective values reached by different methods for 20~metric
    bundle adjustment instances after 100~NNLS solver iterations. }
  \label{fig:results_BA_final}
\end{figure}

\subsection{Contrastive Hebbian Learning}
\label{sec:CHL}

Contrastive Hebbian learning uses an energy model over latent variables to
explicitly infer (i.e.\ minimize over) the network activations (instead of
using a predefined rule such as in feed-forward DNNs). Feed-forward DNNs using
certain activation functions can be identified as limit case of suitable
energy-based
models~\cite{xie2003equivalence,scellier2017equilibrium,zach2019contrastive}. We
use the formulation proposed in~\cite{zach2019contrastive} due to the
underlying convexity of the energy model. In the following we outline that the
corresponding supervised learning task is an instance of
Eq.~\ref{eq:main_objective}. In contrastive Hebbian learning the activations
for the network are inferred in two phases: the \emph{clamped phase} uses
information from the target label (via a loss function $\ell$ that is convex
in its first argument) to steer the output layer, and the \emph{free phase}
does not put any constraint on the output. The input layer is always clamped
to the provided training input. The clamped network energy is given
by\footnote{We omit the explicit feedback parameter used
  in~\cite{xie2003equivalence,zach2019contrastive}, since it can be absorbed
  into the activations and network weights.}
\begin{align}
  \hat E(z; \theta) &= \ell(a_L; y) + \frac{1}{2} \Norm{z_1 - W_0 x - b_0}^2
                 + \frac{1}{2} \sum\nolimits_{k=1}^{L-2} \Norm{z_{k+1} - W_kz_k - b_k}^2
\end{align}
subject to $z_k \in \mathcal C_k$, where $\mathcal C_k$ is a convex set and
$\theta$ contains all network weights $W_k$ and biases $b_k$. In order to
mimic DNNs with ReLU activations, we choose
$\mathcal C_k = \mathbb{R}_{\ge 0}^{n_k}$. The loss function is chosen to be
the Euclidean loss, $\ell(a_L;y) = \norm{a_L-y}^2/2$. The dual network energy
can be derived as
\begin{align}
  \hat E^*(\lambda; \theta) &= -\ell^*(-\lambda_L; y) - \frac{1}{2} \sum\nolimits_{k=1}^{L-1} \Norm{\lambda_k}^2 - \lambda_1^T W_0x + \sum\nolimits_{k=1}^L \lambda_k^T b_{k-1}
\end{align}
subject to $\lambda_k \ge W_k^T \lambda_{k+1}$ for $k=1,\dotsc,L-1$.
If $\ell\equiv 0$, i.e.\ there is no loss on the final layer output, then we
denote the corresponding \emph{free} primal and dual energies by
$\check E$ and $\check E^*$, respectively.
Observe that $\hat E$/$\check E$ are convex w.r.t.\ the network activations
$z$, and $\hat E^*$/$\check E^*$ are concave w.r.t.\ the dual variables
$\lambda$.

\begin{figure}[t]
  \centering
  \includegraphics[width=0.49\textwidth]{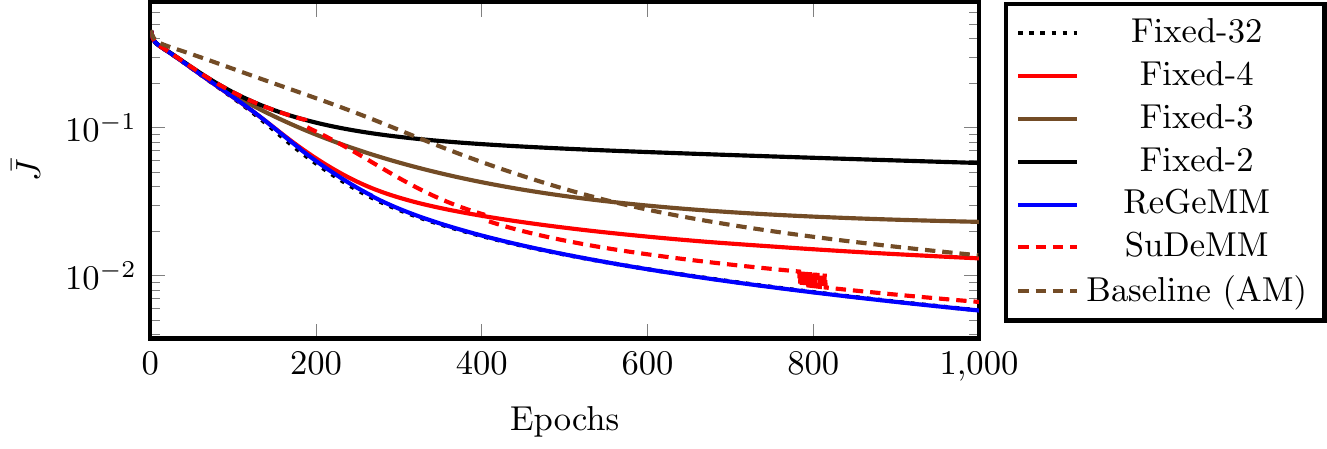}
  \includegraphics[width=0.49\textwidth]{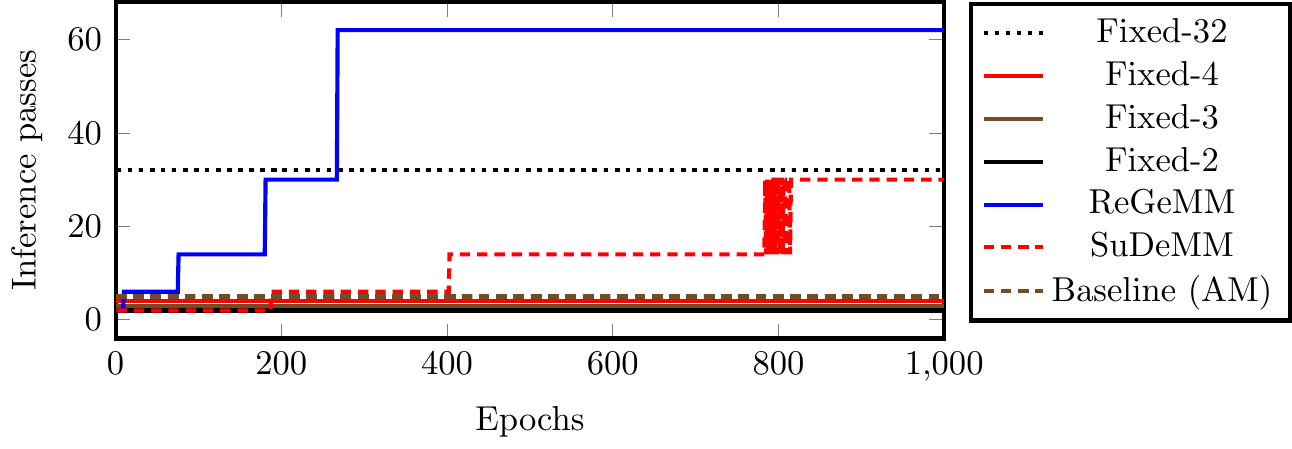}
  \caption{Objective value w.r.t.\ number of epochs (left) and the (accumulated)
    number of inference steps needed to meet the respective criterion
    (right) in the full-batch setting.}
  \label{fig:results_CHL}
\end{figure}

\paragraph{Training using contrastive learning}
Let $\{ (x_i,y_i) \}_i$ be a labeled dataset containing $N$ training samples,
and the task for the network is to predict $y_i$ from given $x_i$. The
utilized contrastive training loss is given by
\begin{align}
  J(\theta) &:= \sum\nolimits_i \left( \min_{\hat z} \hat E(\hat z; x_i, y_i, \theta) - \min_{\check z} \check E(\check z; x_i, \theta) \right) \nonumber \\
            &= \sum\nolimits_i \min_{\hat z} \max_{\check z} \left( \hat E(\hat z; x_i, y_i, \theta) - \check E(\check z; x_i, \theta) \right).
\end{align}
which is minimized w.r.t.\ the network parameters $\theta$. Using duality this
saddlepoint problem can be restated as pure minimization and maximization tasks~\cite{zach2019contrastive},
\begin{align}
  \ub{J}(\theta,(\hat z_i,\check\lambda_i)_{i=1}^N) &= \sum\nolimits_i \left( \hat E(\hat z_i; x_i, y_i, \theta) - \check E^*(\check\lambda_i; x_i, \theta) \right). \\
  \lb{J}(\theta,(\hat\lambda_i,\check z_i)_{i=1}^N) &= \sum\nolimits_i \left( \hat E^*(\hat\lambda_i; x_i, y_i, \theta) - \check E(\check z_i; x_i, \theta) \right).
\end{align}
Thus, the latent variables $\ub{u}_i=(\hat z_i, \check\lambda_i)$ and
$\lb{u}_i=(\hat\lambda_i, \check z_i)$ correspond to primal-dual pairs
representing the network activations, and therefore the entire set of latent
variables $\ub{\vu}$ is very high-dimensional. In this scenario the true cost
$J$ is not accessible, since it requires solving a inner minimization problem
w.r.t.\ $\ub{u}_i$ not having a closed form solution (it requires solving a
convex QP). Inference (minimization) w.r.t.\ $\ub{u}_i$ is conducted by
coordinate descent, which is guaranteed to converge to a global solution as
both $\hat E$ and $-\check E^*$ are strongly
convex~\cite{luo1992convergence,wright2015coordinate}.

\paragraph{Full batch methods}
In Fig.~\ref{fig:results_CHL} we illustrate the evolution of $J$ on a subset
of MNIST~\cite{lecun1998gradient} using a fully connected
784-64($\times 4$)-10 architecture for 4 methods: (i) inferring
$\ub{u}_i$ with a fixed number of 2, 3, 4 and 32 passes of an iterative
method, respectively, (ii) using the ReGeMM condition Eq.~\ref{eq:relaxed_MM},
and (iii) using the SuDeMM criterion Eq.~\ref{eq:gradient_condition}.
Inference for $\ub{u}_i$ is continued until the respective criterion is
met. Both ReGeMM and SuDeMM use the respective constant memory variants.  In
this scenario 32 passes are considered sufficient to perform inference, and
the ReGeMM and SuDeMM methods track the best curve well. We chose to use the
number of epochs (i.e.\ the number of updates of $\theta$) on the
$x$-axis to align the curves. Clearly, using a fixed number of 2 passes is
significantly faster than using 32 or an adaptive but growing number of
inference steps. Interestingly, the necessary inference steps grow much
quicker (to the allowed maximum of~40 passes) for the ReGeMM condition
compared to the SuDeMM test. The baseline method alternates between gradient
updates w.r.t.\ $\ub{\vu}$ (using line search) and $\theta$. In all methods
the gradient update for $\theta$ uses the same fixed learning rate.

\paragraph{Stochastic methods}

\def\graphicswidth{0.49\textwidth}
\begin{figure}[t]
  \centering
  \includegraphics[width=\graphicswidth]{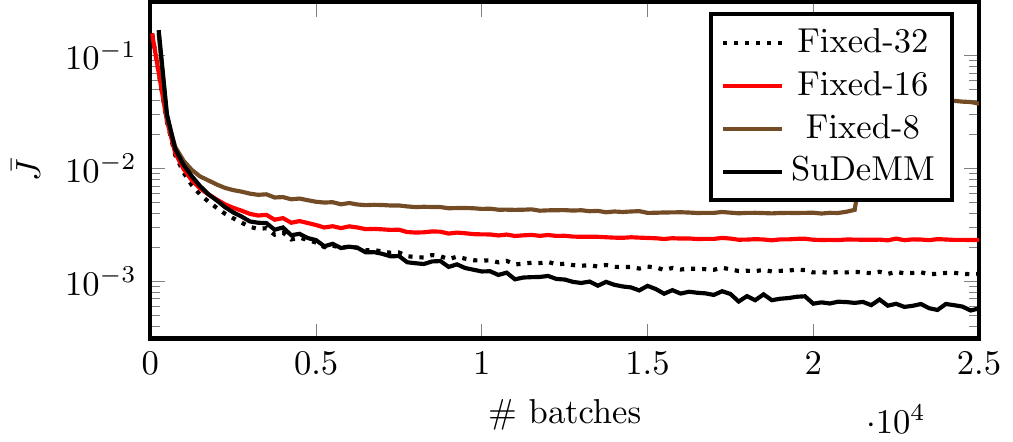}~
  \includegraphics[width=\graphicswidth]{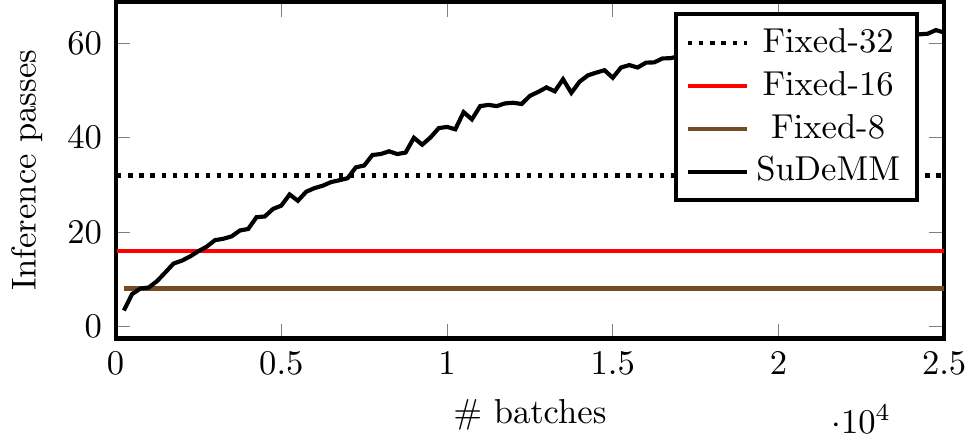}

  \includegraphics[width=\graphicswidth]{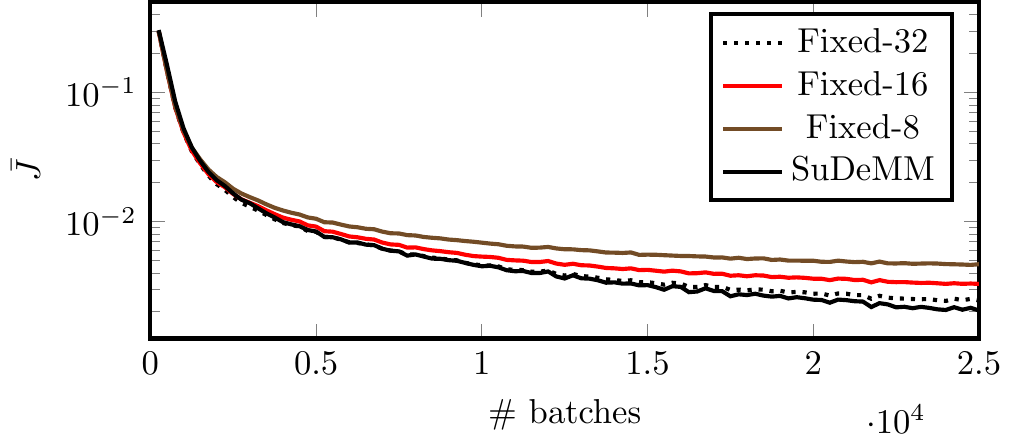}~
  \includegraphics[width=\graphicswidth]{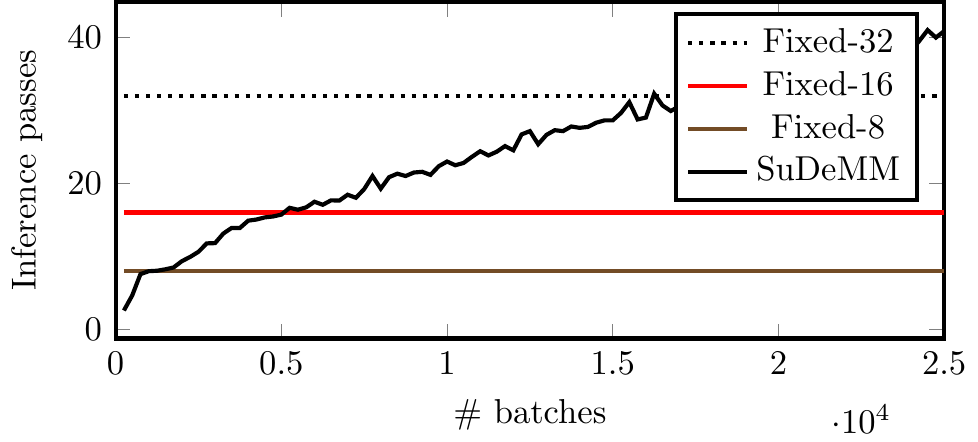}

  \includegraphics[width=\graphicswidth]{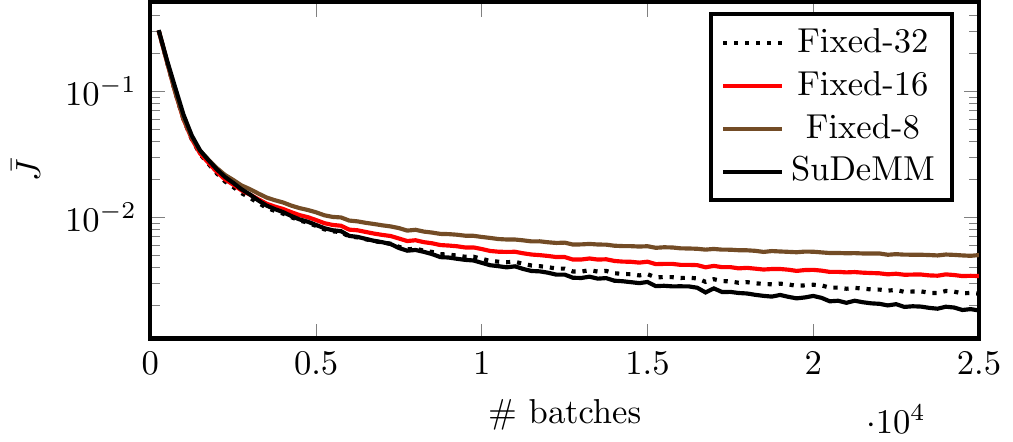}~
  \includegraphics[width=\graphicswidth]{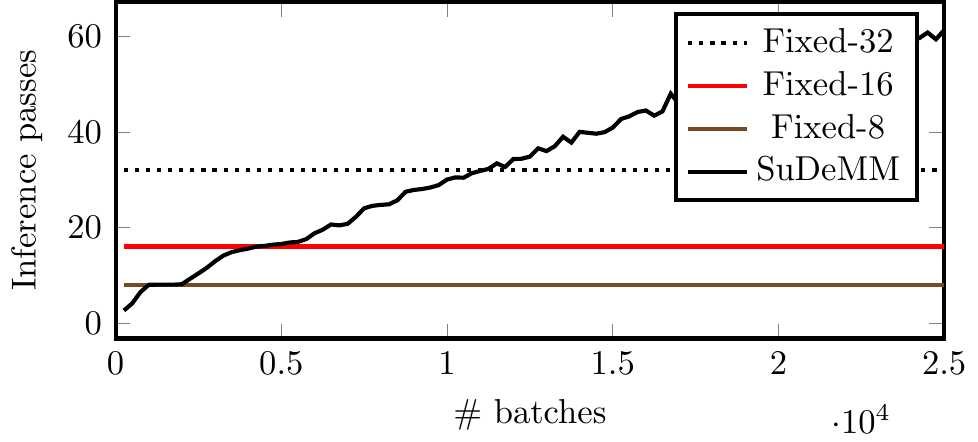}

  \includegraphics[width=\graphicswidth]{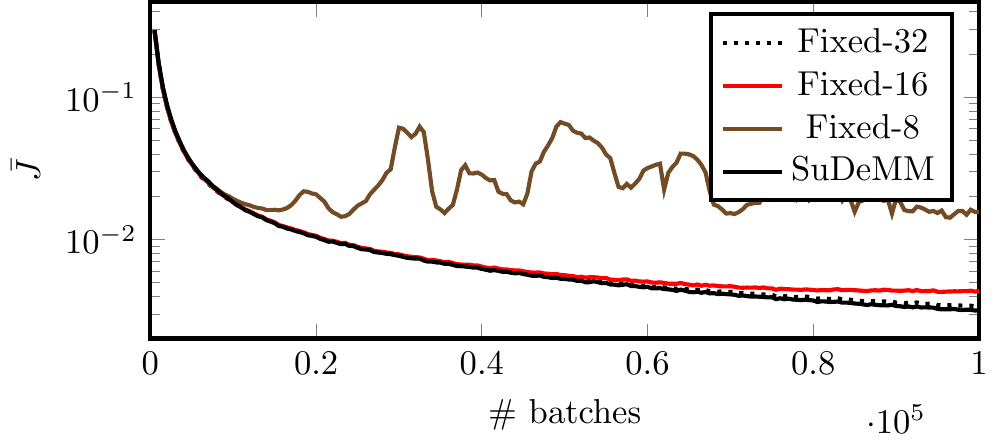}~
  \includegraphics[width=\graphicswidth]{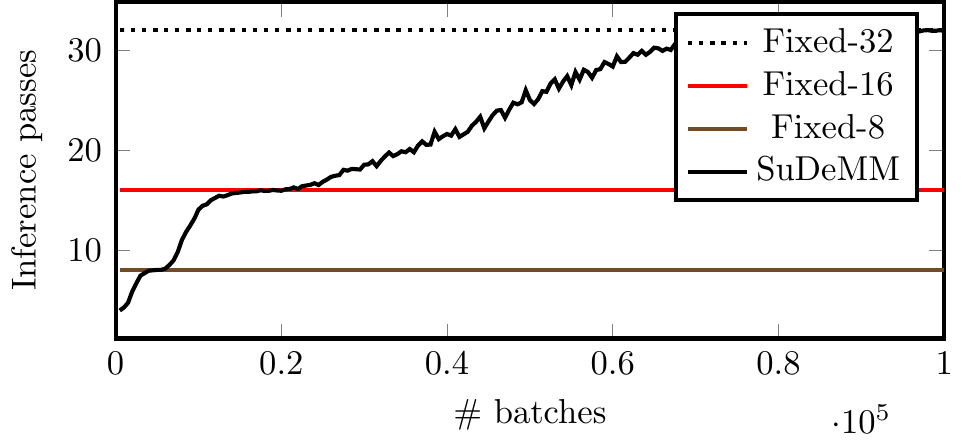}

  \caption{Objective value w.r.t.\ number of processed mini-batches (left column) and the
    number of inference steps needed to meet the respective criterion (right
    column) for MNIST (1st row), Fashion-MNIST (2nd row), KMNIST (3rd row) and
    CIFAR-10 (bottom row) using the stochastic gradient method.}
  \label{fig:results_mnist_SGD}
\end{figure}

For the stochastic method in Alg.~\ref{alg:SuDeMM_stochastic} we illustrate
the evolution of the objectives values $\ub{J}$ and the number of inference
passes in Fig.~\ref{fig:results_mnist_SGD}. For MNIST and its drop-in
replacements Fashion-MNIST~\cite{xiao2017fashion_mnist} and
KMNIST~\cite{clanuwat2018deep} we again use the same
784-64($\times 4$)-10 architecture as above. For a greyscale version of
CIFAR-10~\cite{krizhevsky2009learning} we employ a
1024-128($\times 3$)-10 network. The batch size is 10, and a constant step
size is employed. The overall conclusions from
Fig.~\ref{fig:results_mnist_SGD} are as follows: using an insufficient number
of inference passes yields poor surrogates for the true objective $J$ and it
can lead to numerical instabilities due to the biasedness of the gradient
estimates. Further, the proposed SuDeMM algorithm yields the lowest estimates
for the true objective by gradually adapting the necessary inference
precision.

\section{Conclusion}

We present two approaches to optimize problems with a latent variable
structure. Our formally justified methods (i) enable inexact (or truncated)
minimization over the latent variables and (ii) allow to discard the latent
variables between updates of the main parameters of interest. Hence, the
proposed methods significantly reduce the memory consumption, and
automatically adjust the necessary precision for latent variable inference.
One of the two presented methods can be adapted to return competitive
solutions for highly non-convex problems such as large-scale robust
estimation, and the second method can be run in a stochastic optimization
setting in order to address machine learning tasks.

In the future we plan to better understand how turning the proposed ReGeMM
inequality condition essentially into an equality constraint can help with
solving highly non-convex optimization problems. Further, the presented SuDeMM
method enables us to better explore a variety of convex energy-based models in
the future.

\appendix

\section{Relaxed Generalized MM and the gradient method}

We assume that $\ub{J}(\cdot;\ub{u})$ has a Lipschitz gradient with constant
$L$, and therefore
\begin{align}
  \ub{J}(\theta; \ub{u}) &\le \ub{J}(\theta^0; \ub{u}) + \nabla\ub{J}(\theta^0; \ub{u})^T(\theta - \theta^0) + \frac{L}{2} \norm{\theta-\theta^0}^2
\end{align}
for all $\ub{u}$, $\theta$ and $\theta^0$. We recall the relaxed GMM
condition,
\begin{align}
  \ub{J}(\theta^{(t-1)}; \ub{u}^{(t)}) &\le \eta \lb{J}(\theta^{(t-1)}, \lb{u}^{(t)})
  + (1-\eta) \ub{J}(\theta^{(t-2)}; \ub{u}^{(t-1)}) \nonumber \\
  &= \ub{J}(\theta^{(t-2)}; \ub{u}^{(t-1)})
  - \eta \underbrace{\left( \ub{J}(\theta^{(t-2)}, \ub{u}^{(t-1)}) - \lb{J}(\theta^{(t-1)}, \lb{u}^{(t)}) \right)}_{=: c_t} \nonumber.
\end{align}
If $\theta^{(t-1)} = \theta^{(t-2)} - \frac{1}{L} \nabla\ub{J}(\theta^{(t-2)}; \ub{u}^{(t-1)})$, then
\begin{align}
  \ub{J}(\theta^{(t-1)}; \ub{u}^{(t-1)}) &\le \ub{J}(\theta^{(t-2)}; \ub{u}^{(t-1)})
  + \nabla\ub{J}(\theta^{(t-2)}; \ub{u}^{(t-1)})^T(\theta^{(t-1)} - \theta^{(t-2)}) \nonumber \\
  &+ \frac{L}{2} \norm{\theta^{(t-1)}-\theta^{(t-1)}}^2 \nonumber\\
  &= \ub{J}(\theta^{(t-2)}; \ub{u}^{(t-1)}) - \frac{1}{2L} \norm{\nabla\ub{J}(\theta^{(t-2)}; \ub{u}^{(t-1)})}^2.
  \label{eq:lipschitz}
\end{align}
Consequently,
\begin{align}
  c_t &= \ub{J}(\theta^{(t-2)}, \ub{u}^{(t-1)}) - \lb{J}(\theta^{(t-1)}, \lb{u}^{(t)}) \nonumber \\
  &\ge \ub{J}(\theta^{(t-1)}; \ub{u}^{(t-1)}) + \frac{1}{2L} \norm{\nabla\ub{J}(\theta^{(t-2)}; \ub{u}^{(t-1)})}^2
  - \lb{J}(\theta^{(t-1)}, \lb{u}^{(t)}) \nonumber \\
  &\ge J(\theta^{(t-1)}) - \lb{J}(\theta^{(t-1)}, \lb{u}^{(t)}) + \frac{1}{2L} \norm{\nabla\ub{J}(\theta^{(t-2)}; \ub{u}^{(t-1)})}^2 \nonumber \\
  &= \frac{1}{2L} \norm{\nabla\ub{J}(\theta^{(t-2)}; \ub{u}^{(t-1)})}^2 \ge 0,
\end{align}
where we first used Eq.~\ref{eq:lipschitz} and then the lower and upper bounds
on $J$. From Proposition~1 we know that $c_t\to 0$, which implies that
$\norm{\nabla\ub{J}(\theta^{(t-2)}; \ub{u}^{(t-1)})}^2/(2L) \le c_t \to
0$. Hence, limit points of $(\theta^{(t)})_{t=1}^T$ are stationary points of
$J$.

\begin{remark}
  $\theta^{(t)}$ is not necessarily induced by a gradient step, but a new
  iterate $\theta^{(t)}$ has to satisfy a sufficient descent condition,
  \begin{align}
    \ub{J}(\theta^{(t)}; \ub{u}^{(t)}) &\le \ub{J}(\theta^{(t-1)}; \ub{u}^{(t)}) - \frac{\kappa\norm{\nabla\ub{J}(\theta^{(t-1)}; \ub{u}^{(t)})}^2}{2L} \nonumber.
  \end{align}
  for a factor $\kappa\in(0,1)$. In this setting we obtain analogously
  \begin{align}
    c_t \ge \frac{\kappa}{2L} \norm{\nabla\ub{J}(\theta^{(t-2)}; \ub{u}^{(t-1)})}^2 \ge 0,
  \end{align}
  leading to the same conclusion.
\end{remark}

\section{ReGeMM using constant memory}

\begin{algorithm}[tb]
  \begin{algorithmic}[1]
    \Require Initial $\theta^{(0)}=\theta^{(-1)}$ and $\ub{u}^{(0)}$; number of rounds $T$; $R_{\min} \ge 1$
    \For{$t=1,\dotsc,T$}
    \State $R \gets R_{\min}$
    \Repeat
    \State $J_0\gets 0$, $J_1 \gets 0$, $g \gets 0$
    \For{$i=1,\dotsc,N$}
    \State Set
    \begin{align}
      \ub{v} &\gets R\text{-step-arg-min}_{\ub{v}} \ub{J}_i(\theta^{(t-1)}, \ub{v}) & \lb{v} &\gets R\text{-step-arg-max}_{\lb{v}} \lb{J}_i(\theta^{(t-1)}, \lb{v}) \nonumber \\
      J_0 &\gets J_0 + \ub{J}_i(\theta^{(t-1)}, \ub{v}) & J_1 &\gets J_1 + \lb{J}_i(\theta^{(t-1)}, \lb{v}) \nonumber \\
      g &\gets g + \nabla_\theta \ub{J}_i(\theta^{(t-1)}, \ub{v})
    \end{align}
    \EndFor
    \State $R\gets 2R$
    \Until{$J_0 \le \eta J_1 + (1-\eta) J_2$} \Comment{Test for the ReGeMM condition} %\Comment{Test for Eq.~\ref{eq:relaxed_MM}}
    \State Set $\theta^{(t)} \gets \theta^{(t-1)} - \frac{1}{L} g$
    \EndFor
    \State \Return $\theta^{(T)}$
  \end{algorithmic}
  \caption{Relaxed Majorization-Minimization: constant memory version}
  \label{alg:relaxed_MM_constant_memory}
\end{algorithm}

As pointed out in the main text, naive implementations of the ReGeMM algorithm
(Alg.~1 in the main text) require $O(N)$ memory to store
$\ub{\vu}=(\ub{u}_1,\dotsc,\ub{u}_N)$. In many applications the number of
terms $N$ is large, but the latent variables $(\ub{u}_i)_i$ have the same
structure for all $i$ (e.g.\ $\ub{u}_i$ represent pixel-level predictions for
training images of the same dimensions). If we use a gradient method to update
$\theta$, then the required quantities can be accumulated in-place, in
Alg.~\ref{alg:relaxed_MM_constant_memory}.
The constant memory algorithm is not limited to first order methods for
$\theta$, but any method that accumulates the information needed to determine
$\theta^{(t)}$ from $\theta^{(t-1)}$ in-place is feasible (such as the Newton
or the Gauss-Newton method).

The latent variables $\ub{v}$ and $\lb{v}$ are reused for all terms $i$, and
inference for $\ub{v}$ and $\lb{v}$ is conducted using $R$ steps of suitable
iterative inference methods, that are monotonically decreasing and increasing,
respectively. Inference is started from constant inital values for $\ub{v}$
and $\lb{v}$. Suitable choices such for inference methods include gradient
methods and (block) coordinate descent (if the mappings $\ub{\vu}\mapsto
\ub{J}(\theta, \ub{\vu})$ and $\lb{\vu}\mapsto \lb{J}(\theta, \lb{\vu})$ are
strictly convex and concave, respectively, for all $\theta$). If the ReGeMM condition
(Eq.~\ref{eq:relaxed_MM})
is not satisfied, then repeated doubling of $R$ ensures, that the total number
of steps spent for inference is at most four times the minimally required
number of steps.\footnote{In the worst case that the mininum number of steps
  required to meet the relaxed MM condition is $2^M+1$ for some
  $M\in\mathbb{N}$, hence the doubling approach will need
  $1+2+\cdots+2^{M+1} \approx 2^{M+2}$ total inference steps. }

The constant memory algorithm is not limited to first order methods to update
$\theta$, but any method that accumulates the information needed to determine
$\theta^{(t)}$ from $\theta^{(t-1)}$ in-place is feasible (such as the Newton
or the Gauss-Newton method).

\section{Analysis of stochastic SuDeMM}

We use the following assumptions:
\begin{enumerate}
\item $\ub{J}_i(\cdot, \ub{u})$ has Lipschitz gradient with constant
  $L$ for all $i$ and $\ub{u}$. This implies that $J$ has Lipschitz gradient
  as well.
\item All iterates $\theta^{(t)}$, $t\in\mathbb{N}$, are bounded. Together
  with the Lipschitz gradient assumption this means, that the sequence of
  gradients $(g_t)_{t=1}^\infty$ is contained in a bounded set.
\end{enumerate}

We partially follow~\cite{bottou2018optimization}. Let
$g_t := \nabla \ub{J}_{i_t}(\theta^{(t-1)}, \ub{u}_{i_t}^{(t)})$ and therefore
$\theta^{(t)} = \theta^{(t-1)} - \alpha_t g_t$. Thus, we have
\begin{align}
  \ub{J}(\theta^{(t)}, \ub{u}^{(t)}) &= \ub{J}(\theta^{(t-1)} - \alpha_t g_t, \ub{u}^{(t)}) \nonumber \\
  {} &\le \ub{J}(\theta^{(t-1)}, \ub{u}^{(t)}) - \alpha_t \nabla_\theta \ub{J}(\theta^{(t-1)}, \ub{u}^{(t)})^T g_t + \frac{L\alpha_t^2}{2} \norm{g_t}^2.
       \label{eq:Jt_reduction}
\end{align}
The SuDeMM condition implies that
\begin{align}
  \ub{J}_{i_t}(\theta^{(t-1)}, \ub{u}^{(t)}_{i_t}) &\le J_{i_t}(\theta^{(t-1)}) + \frac{\rho_t}{2} \norm{g_t}^2.
\end{align}
By taking the expectation on both sides of this relation we consequently obtain
\begin{align}
  \EE{}{\ub{J}(\theta^{(t-1)}, \ub{u}^{(t)})} &\le J(\theta^{(t-1)}) + \EE{}{\frac{\rho_t}{2} \norm{g_t}^2} \nonumber \\
  {} &\le \ub{J}(\theta^{(t-1)},\ub{u}^{(t-1)}) + \EE{}{\frac{\rho_t}{2} \norm{g_t}^2}.
\end{align}
Combining this with Eq.~\ref{eq:Jt_reduction} yields
\begin{align}
  \EE{}{\ub{J}(\theta^{(t)}, \ub{u}^{(t)})} \le \ub{J}(\theta^{(t-1)},\ub{u}^{(t-1)})
  + \EE{}{\frac{L\alpha_t^2+\rho_t}{2} \norm{g_t}^2 - \alpha_t \nabla_\theta \ub{J}(\theta^{(t-1)}, \ub{u}^{(t)})^T g_t} \nonumber \\
  = \ub{J}(\theta^{(t-1)},\ub{u}^{(t-1)}) + \EE{}{\frac{L\alpha_t^2+\rho_t}{2} \norm{g_t}^2} - \alpha_t \norm{\nabla_\theta \ub{J}(\theta^{(t-1)}, \ub{u}^{(t)})}^2,
  \label{eq:Jt_reduction2}
\end{align}
since $\EE{}{g_t}=\nabla_\theta \ub{J}(\theta^{(t-1)}, \ub{u}^{(t)})$. We
consider the telescopic sum and obtain
\begin{align}
  \EE{}{\ub{J}(\theta^{(T)}, \ub{u}^{(T)})} &- \ub{J}(\theta^{(0)}, \ub{u}^{(0)}) = \EE{}{\sum_{t=1}^T \left( \ub{J}(\theta^{(t)}, \ub{u}^{(t)}) - \ub{J}(\theta^{(t-1)}, \ub{u}^{(t-1)}) \right)} \nonumber\\
  {} &\le \sum_{t=1}^T \left( - \alpha_t \norm{\nabla_\theta \ub{J}(\theta^{(t-1)}, \ub{u}^{(t)})}^2 + \EE{}{\frac{L\alpha_t^2+\rho_t}{2} \norm{g_t}^2} \right) \nonumber
\end{align}
or
\begin{align}
  \EE{}{\sum_{t=1}^T \alpha_t \norm{\nabla_\theta \ub{J}(\theta^{(t-1)}, \ub{u}^{(t)})}^2} &\le \ub{J}(\theta^{(0)}, \ub{u}^{(0)}) - \EE{}{\ub{J}(\theta^{(T)}, \ub{u}^{(T)})}
                                                                                             + \sum_{t=1}^T \EE{}{\frac{L\alpha_t^2+\rho_t}{2} \norm{g_t}^2}.
\end{align}
The r.h.s.\ is finite by our assumptions. With $\sum_{t=1}^\infty\alpha_t =
\infty$ this implies that $\lim\inf_{t\to\infty}\EE{}{\norm{\nabla_\theta
    \ub{J}(\theta^{(t-1)}, \ub{u}^{(t)})}^2} = 0$. With the following simple
lemma we can show something stronger.

\begin{lemma}
  Let $X_t$ be a stochastic process adapted to the filtration $(\cF_t)_t$
  satisfying $\EE{}{X_t-\delta_t|\cF_{t-1}} \le X_{t-1}$ for all
  $t\in\mathbb{N}$. Then $Y_t := X_t - \sum_{r=1}^t\delta_r$ is a
  supermartingale.
\end{lemma}
\begin{proof}
  We have
  \begin{align}
    \EE{}{Y_t|\cF_{t-1}} &= \EE{}{X_t - \sum_{r=1}^t \delta_r |\cF_{t-1}} = \EE{}{X_t-\delta_t |\cF_{t-1}} - \sum_{r=1}^{t-1} \delta_r
                           \le X_{t-1} - \sum_{r=1}^{t-1} \delta_r = Y_{t-1} \nonumber,
  \end{align}
  hence $Y_t$ is a supermartingale.
\end{proof}
We choose $X_t := \ub{J}(\theta^{(t-1)}, \ub{u}^{(t)})$ and
\begin{align}
  \delta_t := \frac{L\alpha_t^2+\rho_t}{2} \norm{g_t}^2 - \alpha_t \norm{\nabla_\theta \ub{J}(\theta^{(t-1)}, \ub{u}^{(t)})}^2.
\end{align}
Using Eq.~\ref{eq:Jt_reduction2} and the above lemma the stochastic process
\begin{align}
  Y_t &:= \ub{J}(\theta^{(t-1)}, \ub{u}^{(t)}) - \sum_{r=1}^t \frac{L\alpha_r^2+\rho_r}{2} \norm{g_r}^2
  + \sum_{r=1}^t \alpha_r \norm{\nabla_\theta \ub{J}(\theta^{(r-1)}, \ub{u}^{(r)})}^2
\end{align}
is a supermartingale. By our assumptions on the sequences $(\alpha_t)_t$,
$(\rho_t)_t$ and $(g_t)_t$, the sum in the middle is bounded. Further, the
first term and the last sum are non-negative. Hence, $Y_t$ is also bounded
from below, and via the supermartingale convergence theorem $Y_t \to
Y_{\infty} < \infty$ a.s. Consequently, the boundedness of
$(\ub{J}(\theta^{(t-1)}, \ub{u}^{(t)}))_t$ and $\sum_{r=1}^\infty
\frac{L\alpha_r^2+\rho_r}{2} \norm{g_r}^2$ implies that
\begin{align}
  \sum_{t=1}^\infty \alpha_t \norm{\nabla_\theta \ub{J}(\theta^{(t-1)}, \ub{u}^{(t)})}^2 < \infty \qquad\text{a.s.}
\end{align}
With $\sum_{t=1}^\infty \alpha_t=\infty$ we deduce that
$\norm{\nabla_\theta \ub{J}(\theta^{(t-1)}, \ub{u}^{(t)})}^2\to 0$ (and
therefore $\nabla_\theta \ub{J}(\theta^{(t-1)}, \ub{u}^{(t)})\to 0$)
a.s. Thus, with probability~1 (w.r.t.\ the sequence of sampled indices
$(i_t)_{t=1}^\infty$) every accumulation point of
$(\theta^{(t)})_{t=1}^\infty$ is a stationary point.

\small
\bibliographystyle{plain}
\bibliography{literature}

\end{document}